\documentclass{spie}

\usepackage[utf8]{inputenc}

\usepackage[table,xcdraw]{xcolor}
\usepackage{times} % Times New Roman
\usepackage{indentfirst} % Indent first paragraphs
\usepackage{multirow}
\usepackage{graphicx}
\usepackage{xkeyval}
\usepackage{adjustbox}
\usepackage{booktabs, multirow} 
\usepackage{soul}
\usepackage[table]{xcolor} % for cell colors
\usepackage{changepage,threeparttable} % for wide tables
\usepackage{amsmath,graphicx, amsfonts}
\usepackage{tabularx}
\usepackage{booktabs, multirow} % for borders and merged ranges
\usepackage{soul}% for underlines
\usepackage[table]{xcolor} % for cell colors
\usepackage{changepage,threeparttable}
\usepackage{algorithm}
\usepackage{algorithmic}
\usepackage{amsmath}
\usepackage{hyperref}

\title{A High-Accuracy Fast Hough Transform with Linear–Log-Cubed Computational Complexity for Arbitrary-Shaped Images}

\author{Danil Kazimirov\supit{1,2,3},
Dmitry Nikolaev\supit{1,4}
\skiplinehalf
  \normalsize
\supit{1} Smart Engines Service LLC, 117312 Moscow, Russia; \\
\supit{2} Institute for Information Transmission Problems (Kharkevich Institute) RAS, 127051 Moscow, Russia; \\
\supit{3} Faculty of Mechanics and Mathematics, Lomonosov Moscow State University, 119991 Moscow, Russia; \\
\supit{4} Federal Research Center Computer Science and Control RAS, 119333 Moscow, Russia.
}

\usepackage{mathptmx}
\begin{document}

\maketitle

\begin{abstract}
The Hough transform (HT) is a fundamental tool across various domains, from classical image analysis to neural networks and tomography. 
Two key aspects of the algorithms for computing the HT are their computational complexity and accuracy -- the latter often defined as the error of approximation of continuous lines by discrete ones within the image region.
%The fast HT (FHT) algorithms with linearithmic complexity are well known; a classic example is the Brady–Yong algorithm, limited to images with power-of-two dimensions. 
%Its generalizations, such as $FHT2DT$, extend applicability to arbitrary image shapes while preserving optimal linearithmic complexity. 
%However, this efficiency comes at the cost of reduced accuracy, which worsens with image size and is significantly lower than that of accurate HT algorithms. 
%They offer constant-bounded approximation error at the expense of near-cubic computational complexity.
The fast HT (FHT) algorithms with optimal linearithmic complexity -- such as the Brady–Yong algorithm for power-of-two-sized images -- are well established. 
Generalizations like $FHT2DT$ extend this efficiency to arbitrary image sizes, but with reduced accuracy that worsens with scale. 
Conversely, accurate HT algorithms achieve constant-bounded error but require near-cubic computational cost.
This paper introduces $FHT2SP$ algorithm -- a fast and highly accurate HT algorithm.
It builds on our development of Brady’s superpixel concept, extending it to arbitrary shapes beyond the original power-of-two square constraint, and integrates it into the $FHT2DT$ algorithm.
%based on the $FHT2DT$ algorithm and enhanced using the concept of superpixels. 
With an appropriate choice of the superpixel's size, for an image of shape $w \times h$, the $FHT2SP$ algorithm achieves near-optimal computational complexity $\mathcal{O}(wh \ln^3 w)$, while keeping the approximation error bounded by a constant independent of image size, and controllable via a meta-parameter. 
We provide theoretical and experimental analyses of the algorithm’s complexity and accuracy.


\end{abstract}

\section{Introduction}
\label{sec:introduction}

The Hough transform (HT), also known as the discrete Radon transform (DRT), is a widely used method in image processing for robust detection of straight lines, based on accumulating pixel values along discrete approximations of continuous lines. The core idea of the HT is the greater the accumulated value along such a line, the more confidently one can infer the presence of the corresponding continuous line in the image.
Originally proposed by Paul Hough in 1959 for identifying particle tracks in bubble chambers~\cite{hough1959machine}, the HT was first applied to localize line segments in images~\cite{rahmdel2015review}. 
Over time, its applications have expanded significantly, now including image binarization~\cite{aliev2014postroenie}, segmentation~\cite{saha2010hough},  text image normalization~\cite{bezmaternykh2024text}, computed tomography~\cite{polevoy2023tomographic}, and even integration into neural network architectures~\cite{sheshkus2020vanishing}.
Due to their importance for computer vision, results related to the HT, including purely theoretical contributions, are periodically presented at the International Conference on Machine Vision (ICMV)~\cite{ershov2015fast, aliev2019use}.

In practical applications, efficient computation of the HT requires both low computational complexity -- defined as the number of arithmetic operations performed, typically correlating with execution time -- and high accuracy, commonly understood as the approximation error of the Radon transform, or the approximation error (for example, maximum orthotropic error) of the continuous lines by discrete ones along which the input image is integrated. 
Regarding computational complexity, several fast HT (FHT) algorithms with linearithmic computational complexity $\mathcal{O}(wh \ln w)$ (as a function of size $wh$ of input image with shape $w \times h$) have been developed, ranging from the classical Brady–Yong algorithm~\cite{brady1992fast} -- applicable only to images with power-of-two dimensions -- to its generalizations for arbitrary image sizes. 
It has been shown that linearithmic complexity is optimal within the class of HT algorithms generalizing the Brady–Yong approach~\cite{kazimirov2025generalizing}, i.e., coinciding with the Brady–Yong algorithm on images with power-of-two dimensions~\cite{khanipov2018computational}. 
The main drawback of such FHT algorithms is their relatively low accuracy, which degrades with increasing image size. 
For instance, the Brady–Yong algorithm approximates continuous lines using so-called dyadic patterns~\cite{brady1992fast}, and for an image of width $w = 2^q$, $q \in \mathbb{N}$, the maximum orthotropic line approximation error is $\frac{\log_2 w}{6}$ for even $q$, and $\frac{\log_2 w}{6} - \frac{1}{18}$ for odd $q$~\cite{karpenko2021analysis, smirnov2023analyzing}.
On the other hand, accurate HT algorithms~\cite{math11153336, khanipov2018ensemble} provide an optimal, constant-bounded approximation error for any input image, but at the cost of significantly higher, often near-cubic, computational complexity.

In this work, we propose a novel $FHT2SP$ algorithm -- highly-accurate HT algorithm that is based on our refinement of Brady’s superpixel concept, generalized to support arbitrary shapes instead of being limited to power-of-two squares, and incorporated into the $FHT2DT$ algorithm. 
The $FHT2SP$ algorithm combines an approximation error bounded by a constant $\lambda + \frac{1}{2}$, independent of input image size and controlled by an algorithm's meta-parameter $\lambda \in (0, 1]$, with near-optimal linear-log-cubed computational complexity $\mathcal{O}(wh \ln^3 w)$ when superpixel's size is appropriately chosen.
%The provided theoretical properties of the $FHT2SP$ algorithm are consistent with the outcomes of our experimental evaluation.
The provided theoretical properties of the $FHT2SP$ algorithm align with experimental results.

\section{Related work}
\label{sec:related_works}

In addition to the Brady–Yong algorithm, the central and periodic DRT algorithms are applicable to images with power-of-two dimensions~\cite{app112210606}.
Both outperform the Brady-Yong algorithm in computational complexity while retaining $\mathcal{O}(wh \ln w)$ asymptotics, with $w$ and $h$ signifying image width and height. 
However, similar to the Brady-Yong algorithm, their approximation error increases with image size. 
For arbitrary-sized images, several fast generalizations of the Brady–Yong algorithm have been proposed—namely, $FHT2DS$, $FHT2DT$, $FHT2SS$, $FHT2ST$, $FHT2MS$, and $FHT2MT$~\cite{kazimirov2025generalizing}. 
All retain linearithmic complexity, though their accuracy also degrades with increasing image size. 
It was shown that the growth of orthotropic approximation error in these six algorithms is at most logarithmic. 
Among them, $FHT2DT$ was identified as best balancing low computational cost, high accuracy, and precision of the adjoint transform. 

In contrast, accurate HT algorithms such as $ASD2$~\cite{math11153336} and $KHM$~\cite{khanipov2018ensemble} leverage digital straight line segments (DSLS)~\cite{rosenfeld2006digital} to limit the maximum orthotropic approximation error to a constant $1/2$, achieving significantly better accuracy than $FHT2DT$’s $(7 + 2\log_2 w)/12$ bound. 
However, this comes at higher computational costs: $\mathcal{O}(n^{8/3})$ for $ASD2$ (with $n = w = h = 8^q$, $q \in \mathbb{Z}$, $q \geq 0$) and $\mathcal{O}(n^3 / \ln n)$ for $KHM$ (when processing an image of arbitrary shape $n \times n$). 

%A direction toward improving Brady–Yong's accuracy was explored by Brady~\cite[Section 3.4]{brady1998fast},
The Brady–Yong algorithm was originally developed jointly by Martin~L.~Brady and Whanki~Yong~\cite{brady1992fast}, and a subsequent refinement aimed at improving its accuracy was proposed by Martin~L.~Brady in a later work~\cite[Section 3.4]{brady1998fast}.
Brady proposed a superpixel-based method where each $1\times1$ pixel is replaced with a $k \times k$ block, $k=2^m$, $m \in \mathbb{Z}$, $m \geq 0$, or superpixel, which is an image with one column (or row) retaining the original pixel value, and others set to zero. 
The superpixel approach enables reduced maximum orthotropic error by a factor of $k$ while 
sharing the same asymptotic computational complexity as the original algorithm, when the superpixel's shape is fixed. 
We observe that Brady’s superpixel strategy can be generalized to other FHT algorithms, and the constraint of square superpixels with power-of-two linear size can be relaxed. 
In the next section, we integrate the extended superpixel structures into the $FHT2DT$ algorithm, resulting in the proposed $FHT2SP$ (SuperPixel-based) method, applicable to arbitrary-shaped image. 
Among the generalizations of the Brady–Yong for arbitrary image sizes, $FHT2DT$ was selected for enhancement as it best meets practical criteria~\cite{kazimirov2025generalizing}. 
%We study the computational complexity and accuracy of $FHT2SP$ both theoretically and experimentally, including the impact of arbitrarily shaped superpixels on its performance.

% Сказать про обобщения, которые применимы к изображениям произвольного размера,  FHT2DS, FHT2DT и пр. -- они имеют оптимльную скорость, но ошибка растет. 

% Среди точных алгоритмов надо выделить Knanipov, ASD2. Их ошибка ограничена 0.5; потому что они используют DSLS (работа Розенфельда).
% Шаг в сторону уменьшения ошибки предпринят Брейди.
% Его идею применяем и мы к FHT2DT. Он выбран, потому что так рекомендовалось в статье IEEE access

\section{Algorithm description}
\label{sec:algorithm_description}

%The $FHT2SP$ algorithm enhances the accuracy of $FHT2DT$ by leveraging the concept of superpixel, which, by definition, is itself an image that is assigned to a single pixel of the original image.
%The superpixel concept is employed in the Brady algorithm~\cite[Section 3.4]{brady1998fast}, which was originally developed exclusively for images with power-of-two dimensions and uses square superpixels with power-of-two dimensions by design.
We describe an $FHT2SP$ HT algorithm that computes integrals along a quarter-set of discretizations of non-decreasing lines with slopes in the interval $[0, 1]$, commonly referred to as predominantly horizontal lines~\cite{math11153336, kazimirov2025generalizing}.
The resulting integral values are stored in the output Hough image: the pixel at position $(t, s)$ contains the integral of the input image along the discrete line parameterized by $(t, s)$, which approximates a continuous line of the form $l(t,s)=l(t,s)(x)= s + \frac{t}{w-1}x$, in a coordinate system with the origin located at the center of the bottom-left image pixel, the $Ox$ and $Oy$ axes are aligned with the image width and height, respectively.

Consider an image $I = I_{w \times h}: \mathbb{Z}_w \times \mathbb{Z}_h \equiv \{0,1,\ldots w-1 \} \times \{0,1,\ldots h-1 \} \rightarrow \mathbb{A}$ of the width $w$ and height $h$, with pixel values contained in Abelian semigroup $(\mathbb{A}, +)$ with a neutral element. Within the $FHT2SP$ algorithm,  each pixel of the input image $I$ is replaced by a superpixel of some (not necessarily square or with power-of-two dimensions, as Brady utilized~\cite{brady1998fast}) shape $\widehat{w} \times \widehat{h}$. 
In the $FHT2SP$ superpixel with indices $(x, y) \in \mathbb{Z}_w \times \mathbb{Z}_h$, the column (with height $\widehat{h}$) corresponding to the index $ \widehat{n} \in \mathbb{Z}_{\widehat{w}}$ is filled with the value $I(x, y)$. 
The remaining pixels of the superpixel with indices $(x, y)$ are filled with zero values. 
After this filling, we obtain the image $\widehat{I} = \widehat{I}_{W \times H}$, where $W = \widehat{w} \cdot w$ and $H = \widehat{h} \cdot h$:
\begin{equation} \label{eq:I_hat}
\widehat{I}(k\widehat{w} + \widehat{n}, m\widehat{h} + y) = I(k, m) \quad \forall k \in \mathbb{Z}_w, m \in \mathbb{Z}_h, y \in \mathbb{Z}_{\widehat{h}},
\end{equation}
while the remaining pixels of $\widehat{I}$ are filled with the value 0.

Then, the $FHT2DT$ algorithm is applied to the image $\widehat{I}$ (Equation~\ref{eq:I_hat}), resulting in the Hough image $\widehat{J} = \widehat{J}_{W \times H}$ of the expanded image $\widehat{I}$.
Next, it is necessary to perform subsampling to obtain the final Hough image $J = J_{w \times h}$ of the input image $I$.
In order to describe the subsampling of $\widehat{J}$, we specify a correspondence between geometric lines within the image $\widehat{I}$ and geometric lines with $st$-parameters $(t, s) \in \mathbb{Z}_w \times \mathbb{Z}_h$ within the image $I$, that is, lines of the form $y = y(x) = \frac{t}{w-1}\left(x - \frac{1}{2}\right) + s + \frac{1}{2}$ in a coordinate system with axes along the sides of the image $I$, where the origin is in the bottom-left corner. 
This will allow us to map the line integrals over the original image $I$ to the corresponding line integrals over $\widehat{I}$.

When mapping a unit square $1 \times 1$ onto a rectangle $\widehat{w} \times \widehat{h}$, the line
%\begin{gather*}
$y = y(x) = \frac{t}{w-1}\left(x - \frac{1}{2}\right) + s + \frac{1}{2}$ % \equiv l(t, s)\left(x - \frac{1}{2} \right) + \frac{1}{2}, 
%\\
%l(x) = l(t,s)(x) = \frac{t}{w-1}x+s,
%\end{gather*} 
transforms into the line
%\begin{equation*}
$\widehat{y} = \widehat{y}(x) = \frac{\widehat{h}}{\widehat{w}} \frac{t}{w-1} \left(x - \frac{\widehat{w}}{2}\right) + \widehat{h} \left(s + \frac{1}{2}\right)$
%\end{equation*}
in the coordinate system with axes along the sides of $\widehat{I}$, where the origin is in the bottom-left corner. 
This line can also be expressed as
\begin{equation} \label{widehat_y_C}
\widehat{y}_{C} = \widehat{y}_{C}(x) = \frac{\widehat{h}}{\widehat{w}} \frac{t}{w-1} \left(x - \frac{\widehat{w}}{2} + \frac{1}{2}\right) + \widehat{h} \left(s + \frac{1}{2}\right) - \frac{1}{2}
\end{equation}
in the coordinate system with axes along the sides of the image $\widehat{I}$, where the origin is at the center of the bottom-left pixel. 
By rounding the coordinates of the endpoints of the line $\widehat{y}_{C}=\widehat{y}_{C}(t,s)$, we get a line, denoted as $\left[ \widehat{y}_{C} \right]$, which connects points with coordinates $\left(0, \left[ \widehat{y}_{L} \right] \right)$ and $\left(w \widehat{w} - 1, \left[ \widehat{y}_{R} \right] \right)$, where $\widehat{y}_L = \widehat{y}_{C}(0)$ and $\widehat{y}_R = \widehat{y}_{C}(w \widehat{w} - 1)$.
%\begin{gather*}
%\widehat{y}_L = \widehat{y}_{C}(0) = \frac{\widehat{h}}{\widehat{w}} \frac{t}{w-1} \left( \frac{1}{2} - \frac{\widehat{w}}{2} \right) + \widehat{h} \left( s + \frac{1}{2} \right) - \frac{1}{2}, \quad
%\widehat{y}_R = \widehat{y}_{C}(w \widehat{w} - 1) = \frac{\widehat{h}}{\widehat{w}} \frac{t}{w-1} \left(w \widehat{w} - \frac{\widehat{w}}{2} - \frac{1}{2}\right) + \widehat{h} \left(s + \frac{1}{2}\right) - \frac{1}{2}.
%\end{gather*}
The notation $[\cdot]$ indicates rounding to the nearest integer value.
Hence, if the lines are considered cyclically extended beyond the image boundaries, the line $\left[ \widehat{y}_{C} \right]$ possesses $(\widehat{t}, \widehat{s})$ coordinates in the $st$-parametrization:
\begin{gather} \label{eq:st_hat}
    \widehat{t} = \widehat{t}(t, s) = \left(  \left[ \widehat{y}_{R} \right] - \left[ \widehat{y}_{L} \right] \right) \mod w \widehat{w}, \quad
    \widehat{s} = \widehat{s}(t,s) = \left[ \widehat{y}_{L} \right] \mod h \widehat{h}, \quad \widehat{y}_L = \widehat{y}_{C}(0), \quad \widehat{y}_R = \widehat{y}_{C}(w \widehat{w} - 1),
\end{gather}
and, consequently, 
\begin{equation} \label{[widehat_y_C]}
    [\widehat{y}_C] = [\widehat{y}_C](x) = \widehat{s} + \frac{\widehat{t}}{w-1}x.
\end{equation}
% [\widehat{y}_C](\widehat{t}, \widehat{s}), but \widehat{y}_C(t, s)

The pixel value $J(t, s)$ is set equal to the value $\widehat{J} \left( \widehat{t}, \widehat{s} \right)$ accumulated along the $FHT2DT$ discrete line, which approximates line $[\widehat{y}_C]=[\widehat{y}_C](\widehat{t},\widehat{s})$ within the expanded image $\widehat{I}$,
%\begin{equation*}
%J(t, s) = \widehat{J} \left( \widehat{t}, \widehat{s} \right),
%\end{equation*}
meaning that we replace the sum of pixel values of the image $\widehat{I}$ along the line $\widehat{y}_C=\widehat{y}_C(t,s)$ with non-integer endpoints with the sum of pixels along the line $[\widehat{y}_C]=[\widehat{y}_C](\widehat{t},\widehat{s})$ having endpoints with the nearest integer coordinates (to the line $\widehat{y}_C$).
%Thus, we obtain the following pseudocode Algorithm~\ref{alg:fht2sp} for the $FHT2SP$ algorithm. 
Thus, we obtain the pseudocode of the $FHT2SP$ algorithm, as presented in Algorithm~\ref{alg:fht2sp}, with its workflow illustrated in Figure~\ref{fig:algorithm_workflow}.
In the pseudocode, the function $Create\_Zeroed\_Image(w, h)$ creates an image of size $w \times h$ initialized with zeros.
For the pseudocode of the $FHT2DT$ algorithm, see papers~\cite{kazimirov2025generalizing, kazimirov2025generalization}.
Note that the $FHT2SP$ algorithm reduces to the $FHT2DT$ algorithm when the superpixel dimensions are set to unity, i.e., $\widehat{w} = \widehat{h} = 1$.

\begin{figure}[!ht]
\centering
\scalebox{0.885}{
\begin{minipage}[c]{0.6\textwidth}
\begin{algorithm}[H]
\caption{Algorithm $FHT2SP$}\label{alg:fht2sp}
    \begin{algorithmic}[1]
    \STATE{\textbf{Input:} image width $w > 0$, image height $h > 0$, image $I = I_{w \times h}$, superpixel width $\widehat{w}$, superpixel height $\widehat{h}$, superpixel non-zero column number $\widehat{n} \in \mathbb{Z}_{\widehat{w}}$}
    \STATE{\textbf{Output:} Hough image $J = J_{w \times h}$}
    \IF{$w > 1$}
        \STATE{$\widehat{I} \leftarrow Create\_Zeroed\_Image \left(w \widehat{w}, h \widehat{h} \right)$}
        \FOR{$k \leftarrow 0$ \textbf{to} $w - 1$}
            \FOR{$m \leftarrow 0$ \textbf{to} $h - 1$}
                \STATE{$\widehat{I}\left(k \widehat{w} + \widehat{n}, m\widehat{h}:(m+1) \widehat{h} \right) \leftarrow I(k, m) $}
            \ENDFOR
        \ENDFOR
        \STATE{$\widehat{J} \leftarrow FHT2DT\left(w \widehat{w}, h \widehat{h}, \widehat{I}\right)$}
        \FOR{$t \leftarrow 0$ \textbf{to} $w - 1$}
            \FOR{$s \leftarrow 0$ \textbf{to} $h - 1$}
                \STATE{$\widehat{y}_{L} \leftarrow \frac{\widehat{h}}{\widehat{w}} \frac{t}{w-1} \left( \frac{1}{2} - \frac{\widehat{w}}{2} \right) + \widehat{h} \left( s + \frac{1}{2} \right) - \frac{1}{2}$}
                \STATE{$\widehat{y}_{R} \leftarrow \frac{\widehat{h}}{\widehat{w}} \frac{t}{w-1} \left(w \widehat{w} - \frac{\widehat{w}}{2} - \frac{1}{2}\right) + \widehat{h} \left(s + \frac{1}{2}\right) - \frac{1}{2}$}
                \STATE{$\widehat{t} \leftarrow \left(  \left[ \widehat{y}_{R} \right] - \left[ \widehat{y}_{L} \right] \right) \mod w \widehat{w}$}
                \STATE{$\widehat{s} \leftarrow \left[ \widehat{y}_{L} \right] \mod h \widehat{h}$}
                \STATE{$J(t,s) \leftarrow \widehat{J}\left(\widehat{t}, \widehat{s}\right)$}
            \ENDFOR
        \ENDFOR
    \ELSE
        \STATE{$J \leftarrow I$}
%        \RETURN{}
    \ENDIF
    \end{algorithmic}
\end{algorithm}
\end{minipage}
}
\hfill
\scalebox{0.855}{
\begin{minipage}[c]{0.49\textwidth}
\centering
\includegraphics[width=\linewidth]{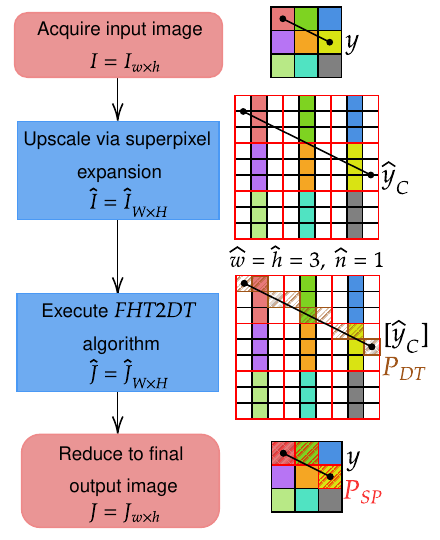}
\caption{Workflow of the $FHT2SP$ algorithm. An example of input image transformations is presented for the case $w=h=\widehat{w}=\widehat{h}=3$, $\widehat{n}=1$. Pixels with identical values are shown in the same color. Each superpixel is delineated with a red border, while the white pixels within superpixels correspond to zero columns.}
\label{fig:algorithm_workflow}
\end{minipage}
}
\end{figure}

\section{Theoretical analysis of complexity and accuracy}
\label{sec:complexity_and_accuracy_theoretical_analysis}

\subsection{Accuracy}~\label{sec:accuracy}
%Let us discuss the accuracy of the $FHT2SP$ algorithm. 
Any $FHT2DT$ discrete line in the region of the expanded image $\widehat{I}$ intersects a non-zero column of the superpixel at a single pixel $(\widehat{x}, \widehat{y})$. 
Since there is a one-to-one correspondence between the superpixels and the pixels of the original image, when adding the value $\widehat{I}(\widehat{x}, \widehat{y})$ to the sum $\widehat{J}\left( \widehat{t}, \widehat{s} \right)$, we essentially add $I\left(\big\lfloor \frac{\widehat{x}}{\widehat{w}} \big\rfloor, \big\lfloor \frac{\widehat{y}}{\widehat{h}}\big \rfloor \right)$. 
In other words, $J(t, s)$ is the sum along the $FHT2SP$ discrete line $P_{SP}(t, s) \in \mathcal{P}_{SP}(w, h, \widehat{w}, \widehat{h}, \widehat{n})$ of the form:
\begin{gather} \label{eq:P_SP}
    P_{SP}(t,s) = \Bigg\{ \left( x, P_{SP}(t, s)(x) = \Big\lfloor \frac{P_{DT}(\widehat{t}, \widehat{s})(x\widehat{w}+\widehat{n})}{\widehat{h}} \Big \rfloor  \right)\, \Big| \, x \in \mathbb{Z}_{w} \Bigg \} \in \mathcal{P}_{SP}(w, h, \widehat{w}, \widehat{h}, \widehat{n}),\\
    \label{eq:P_DT}
    P_{DT}(\widehat{t}, \widehat{s})=\Big\{ \left(\widehat{x}, P_{DT}(\widehat{t}, \widehat{s})(\widehat{x}) \right)\, | \, \widehat{x} \in \mathbb{Z}_{w \widehat{w}} \Big \} \in \mathcal{P}_{DT}(w\widehat{w}, h\widehat{h}), \; t \in \mathbb{Z}_w, s \in \mathbb{Z}_h,
\end{gather}
with parameter values $(\widehat{t}, \widehat{s})$ defined by Equations~\ref{eq:st_hat}.
Here, in the expression~\ref{eq:P_DT}, $\mathcal{P}_{DT}\left(w \widehat{w}, h \widehat{h} \right)$ denotes the set of discrete lines implicitly constructed by the $FHT2DT$ algorithm within the image region of shape $w \widehat{w} \times h \widehat{h}$~\cite{kazimirov2025generalizing}.
Hence, equivalently, we can express the $FHT2SP$ Hough image values as 
%\begin{equation*}
    $J(t, s) = \sum_{(x,y) \in P_{SP}(t,s)} I(x,y)$.
%\end{equation*}
The $FHT2DT$ discrete lines are presented by so-called patterns, which are discretely continuous discrete lines, i.e. with no jumps of magnitude bigger than 1~\cite{kazimirov2025generalizing, kazimirov2025generalization, math11153336}.
Therefore, the Equation~\ref{eq:P_SP} confirms that, for $\widehat{w} \leq \widehat{h}$, the $FHT2SP$ discrete lines are also patterns.

%The HT computed by the $FHT2SP$ algorithm integrates the input image along discrete lines defined by the $FHT2SP$ patterns $\mathcal{P}_{SP}(w, h, \widehat{w}, \widehat{h}, \widehat{n})$.
The accuracy of the patterns-based FHT is commonly measured as the maximum orthotropic approximation error of continuous lines by patterns~\cite{karpenko2021analysis, brady1998fast, kazimirov2025generalizing}.
Consequently, the accuracy, or approximation error, $\mathcal{E}_{SP}(w, h, \widehat{w}, \widehat{h}, \widehat{n})$ of the $FHT2SP$ algorithm can be calculated as given by equation
%\begin{equation*}
    $\mathcal{E}_{SP}(w, h, \widehat{w}, \widehat{h}, \widehat{n}) = \max_{(t, s) \in \mathbb{Z}_w \times \mathbb{Z}_h} \|l(t,s) - P_{SP}(t,s) \|$,
%\end{equation*}
where $l(t,s) = \Big\{ \left(x, l(t,s)(x) = s + \frac{t}{w-1}x \right) \big| x \in \mathbb{Z}_w \Big\}$, $P_{SP}(t, s) \in \mathcal{P}_{SP}(w, h, \widehat{w}, \widehat{h}, \widehat{n})$ and $\|f-g \| = \max_{x \in \mathbb{Z}_w}|f(x) - g(x)|$. 
We also recall here that $\mathcal{E}_{DT}(w, h) = \max_{(t, s) \in \mathbb{Z}_w \times \mathbb{Z}_h} \|l(t,s) - P_{DT}(t,s) \| \leq \frac{\log_2 w}{6} + \frac{7}{12}$~\cite{kazimirov2025generalizing}.

We establish the theorem concerning the accuracy of the $FHT2SP$ algorithm.
\begin{theorem} \label{thm:fht2sp_accuracy}
    Let $\widehat{w}$ and $\widehat{h}$ be odd integers, and let $\widehat{n} = \frac{\widehat{w}-1}{2}$. If the inequality $2 \log_2 (w \widehat{w}) + 13 < 12 \lambda \widehat{h}$, $0 < \lambda \leq 1$, is satisfied, then the orthotropic approximation error bound $\mathcal{E}_{SP}(w, h, \widehat{w}, \widehat{h}, \widehat{n}) < \lambda + \frac{1}{2}$ holds.
    %\begin{equation} \label{eq:condition_for_accuracy_bound}
    %$2 \log_2 (w \widehat{w}) + 13 < 12 \lambda \widehat{h}, \quad 0 < \lambda \leq 1$,
    %\end{equation}
    %then 
    %\begin{equation*}
    %$\mathcal{E}_{SP}(w, h, \widehat{w}, \widehat{h}, \widehat{n}) < \lambda + \frac{1}{2}$.
    %\end{equation*}
\end{theorem}
\begin{proof}
Fix a pair of parameters $(t, s) \in \mathbb{Z}_{w} \times \mathbb{Z}_h$ and consider lines $\widehat{y}_C=\widehat{y}_C(t,s)$ and $[\widehat{y}_C]=[\widehat{y}_C]\left(\widehat{t},\widehat{s}\right)$ determined by Equations~\ref{widehat_y_C} and \ref{[widehat_y_C]} in the coordinate system with axes along the sides of the image $\widehat{I}$, where the origin is at the center of the bottom-left pixel. 
As $[\widehat{y}_C]$ is obtained from $\widehat{y}_C$ by rounding the endpoints coordinates, one may write $\|\widehat{y}_C - [\widehat{y}_C] \|_{\widehat{I}} \leq \frac{1}{2}$; index $\widehat{I}$ means that the norm is calculated over the geometric region of the image $\widehat{I}$. 

The $FHT2SP$ pattern $P_{SP}(t,s) \in \mathcal{P}_{SP}(w,h,\widehat{w},\widehat{h}, \widehat{n})$ is constructed by subsampling the $FHT2DT$ pattern $P_{DT}\left(\widehat{t}, \widehat{s} \right)$.
By virtue of the triangle inequality for the norm $\| \cdot \|_{\widehat{I}}$:
\begin{gather} 
    \|P_{DT}\left(\widehat{t},\widehat{s}\right) - \widehat{y}_C(t,s) \|_{\widehat{I}} \leq \|P_{DT}\left(\widehat{t},\widehat{s}\right) - [\widehat{y}_C](\widehat{t},\widehat{s}) \|_{\widehat{I}} + \|[\widehat{y}_C](\widehat{t},\widehat{s}) - \widehat{y}_C(t,s) \|_{\widehat{I}} \leq \label{eq:estimate_sp1} \\
    \|P_{DT}\left(\widehat{t},\widehat{s}\right) - [\widehat{y}_C](\widehat{t},\widehat{s}) \|_{\widehat{I}} + \frac{1}{2} \leq \mathcal{E}_{DT}(w\widehat{w},h\widehat{h}) + \frac{1}{2} \leq \frac{\log_2 (w\widehat{w})}{6} + \frac{13}{12}. \label{eq:estimate_sp2}
\end{gather}
Here, the bound $\|P_{DT}\left(\widehat{t},\widehat{s}\right) - [\widehat{y}_C](\widehat{t},\widehat{s}) \|_{\widehat{I}} \leq \mathcal{E}_{DT}(w\widehat{w},h\widehat{h}) \leq \frac{\log_2 (w \widehat{w})}{6} + \frac{7}{12}$ is explained by the fact that $P_{DT}\left(\widehat{t},\widehat{s}\right)$ is a $FHT2DT$ pattern approximating the line $[\widehat{y}_C](\widehat{t},\widehat{s})$ in the expanded image $\widehat{I}$.
On the basis of the chain of inequalities~\ref{eq:estimate_sp1} and \ref{eq:estimate_sp2}, the inequality
%\begin{equation*}
    $\frac{\log_2 (w \widehat{w})}{6} + \frac{13}{12} < \lambda \widehat{h} \iff 2 \log_2 (w \widehat{w}) + 13 < 12 \lambda \widehat{h}$
%\end{equation*}
guarantees that the deviation of the pattern $P_{DT}(\widehat{t}, \widehat{s})$ from the line $\widehat{y}_C(t,s)$ is less than $\lambda \widehat{h}$.
As soon as $0 < \lambda \leq 1$, the latter means that the pattern $P_{DT}(\widehat{t}, \widehat{s})$ can intersect central columns (corresponding to index $\widehat{n}$) of only those superpixels which are closest to the superpixels traversed by the line $\widehat{y}_C(t,s)$. 

The bound $\|P_{DT}\left(\widehat{t},\widehat{s}\right) - \widehat{y}_C(t,s) \|_{\widehat{I}} < \lambda \widehat{h}$ implies, taking into account that when mapping the line $l(t, s)$, such that $l(t,s)(x) = \frac{t}{w-1}x+s$, to the line $\widehat{y}_C(t, s)$ the orthotropic approximation error is scaled by $\widehat{h}$, that 
\begin{equation} \label{eq:approx_error}
    \|P_{SP}(t,s) - l(t,s) \|_I \equiv \|P_{SP}(t,s) - y(t,s) \|_I \leq \frac{\|P_{DT}\left(\widehat{t},\widehat{s}\right) - \widehat{y}_C(t,s) \|_{\widehat{I}}}{\widehat{h}}  + \frac{1}{2} < \lambda + \frac{1}{2},
\end{equation}
with $\frac{1}{2}$ appeared in the situation when the pattern $P_{DT}(\widehat{t}, \widehat{s})$ intersects the central column of the superpixel different from that which surrounds the line $\widehat{y}_C(t,s)$. 
The bound in Equation~\ref{eq:approx_error} directly infers, by taking supremum over $t$ ans $s$, that $\mathcal{E}_{SP}(w,h,\widehat{w}, \widehat{h}, \widehat{n}) < \lambda + \frac{1}{2}$.
The Theorem is proved.
\end{proof}

\subsection{Computational complexity}~\label{sec:computational_complexity}
The computational complexity of the HT algorithm is defined as the number of arithmetic operations performed during its execution. 
This complexity accounts only for operations directly involved in the evaluation of the Hough image, while auxiliary operations -- such as those related to index calculations -- are excluded~\cite{kazimirov2025generalizing, math11153336}.
The bound of the $FHT2DT$ computational complexity $T_{DT}(w,h)$, when processing an image of shape $w \times h$, is known~\cite{kazimirov2025generalizing, kazimirov2025generalization}:
\begin{equation} \label{eq:complexity_dt_upper}
T_{DT}(w, h) \leq \frac{81 \log_{17} 2}{17} w h  \log_2 w,
\end{equation}
with multiplicative constant $\frac{81 \log_{17} 2}{17}$ being sharp.
Utilizing the bound ~\ref{eq:complexity_dt_upper}, the computational complexity $T_{SP}(w, h, \widehat{w}, \widehat{h})=T_{SP}(w, h, \widehat{w}, \widehat{h}, \widehat{n})$ of the $FHT2SP$ algorithm, as a function of the input image shape $w \times h$ and the superpixel dimensions $\widehat{w} \times \widehat{h}$ (does not depend on index $\widehat{n}$), is estimated as follows:
\begin{theorem} \label{thm:comp_complexity}
    The computational complexity $T_{SP}(w, h, \widehat{w}, \widehat{h})$ of the $FHT2SP$ algorithm possesses the following estimation from above:
    \begin{equation} \label{eq:sp_complexity}
        T_{SP}(w, h, \widehat{w}, \widehat{h}) \leq \frac{81 \log_{17} 2}{17} w\widehat{w}h \widehat{h} \log_2 \left( w \widehat{w} \right).
    \end{equation}

    The estimate is sharp, i.e., the multiplicative constant $\frac{81 \log_{17} 2}{17}$ cannot be reduced:
    $
        \sup_{w, \widehat{w}, h, \widehat{h}} \frac{T_{SP}(w, h, \widehat{w}, \widehat{h})}{w\widehat{w}h \widehat{h} \log_2 \left( w \widehat{w} \right)} = \frac{81 \log_{17} 2}{17}.
    $
\end{theorem}
\begin{proof}
Since the $FHT2SP$ algorithm returns (after subsampling) the Hough image of the expanded image $\widehat{I}$ of size $w\widehat{w} \times h\widehat{h}$ computed via the $FHT2DT$ algorithm, we can apply the complexity estimate Equation~\ref{eq:complexity_dt_upper} for the $FHT2DT$ algorithm to an image $\widehat{I}$ of size $w\widehat{w} \times h\widehat{h}$. 
This proves Equation~\ref{eq:sp_complexity}.
%We obtain exactly:
%\begin{equation}
%    T_{SP}(w, h, \widehat{w}, \widehat{h}) \leq \frac{81 \log_{17} 2}{17} w\widehat{w}h \widehat{h} \log_2 \left( w \widehat{w} \right).
%\end{equation}
The multiplicative constant in the derived inequality is sharp (it cannot be made smaller), as it is sharp in the corresponding inequality for the $FHT2DT$ algorithm (i.e., when $\widehat{w} = \widehat{h} = 1$).
The theorem is proved.
\end{proof}

The next Theorem establishes the $FHT2SP$ computational complexity bound, provided the superpixel dimensions $\widehat{w}=\widehat{w}(w, \lambda)$ and $\widehat{h}=\widehat{h}(w, \lambda)$ are chosen equal positive odd integers and the smallest possible to satisfy the condition in Theorem~\ref{thm:fht2sp_accuracy} and, hence, ensure the maximum orthotropic approximation error bounded by $\lambda + \frac{1}{2}$, $\lambda \in (0, 1]$.

\begin{theorem} \label{thm:linear_log_cubed_complexity}
 Let the value of $\lambda \in (0, 1]$ be fixed, let $\widehat{x}=\widehat{x}(w, \lambda) = \widehat{w}(w, \lambda) = \widehat{h}(w, \lambda)$ be the smallest positive odd integer satisfying the inequality $2 \log_2 (wx) + 13 < 12 \lambda x$, and  $\widehat{n}(w, \lambda)=\frac{\widehat{w}(w, \lambda) - 1}{2}$. Then, the following bound for the computational complexity of the $FHT2SP$ algorithm holds:
 \begin{equation} \label{eq:linear_log_cubed_complexity}
     T_{SP}(w, h, \widehat{x}, \widehat{x}) = \mathcal{O} \left( w h \ln^3 w \right),
 \end{equation}
 i.e., there exists a constant $C=C(\lambda)>0$ such that $T_{SP}(w, h, \widehat{x}, \widehat{x}) \leq C w h \ln^3 w$ for $w \geq w_0 = \text{const} > 0$ and any dynamics of $h=h(w)$.
 \end{theorem}

 \begin{proof}
Due to the definition of  $\widehat{x}$, the estimation
%\begin{equation} %\label{eq:x_hat_estimation}
$0 \leq \widehat{x} - \left( - \frac{1}{6 \lambda \ln 2 } W_{-1}\left(\frac{-3 \lambda \ln 2}{32 \sqrt{2} w} \right) \right) \leq 2$,
%\end{equation}
is valid for all sufficiently large $w$, where $-W_{-1}\left(\frac{-3 \lambda \ln 2}{32 \sqrt{2} w} \right) / (6 \lambda \ln 2)$ is a solution of the equation $2 \log_2 (wx) + 13 = 12 \lambda x$, and $W_{-1}(\cdot)$ signifies a real branch of the Lambert W function defined on the half-open interval $[-\frac{1}{e}, 0)$. 
Since $\frac{e}{e-1} \ln(-x) \leq W_{-1}(x) \leq \ln(-x) - \ln (- \ln (-x)) $ is applicable for $x \in [-1/e, 0)$~\cite{loczi2020explicit}, from the estimation for $\widehat{x}$, we derive $\widehat{x}= \mathcal{O} \left( \ln w  \right)$.
Therefore, the latter, in combination with Theorem~\ref{thm:comp_complexity}, implies
%\begin{equation*}
$T_{SP}(w, h, \widehat{x}, \widehat{x}) \leq \frac{81 \log_{17} 2}{17} wh \widehat{x} \,^2 \log_2 \left( w \widehat{x} \right) < \frac{486 \log_{17} 2}{17} \lambda wh \widehat{x}\,^3 = \mathcal{O}(wh \ln^3 w)$,
%\end{equation*}
which ends the proof of Equation~\ref{eq:linear_log_cubed_complexity}.
\end{proof}

\section{Experiments and discussion}
\label{sec:experiments}

We conducted an experiment aimed at analyzing the computational complexity and accuracy of the proposed $FHT2SP$ algorithm. 
A square superpixel $\widehat{x} \times \widehat{x}$ was used, with linear size $\widehat{x}=\widehat{x}(n, \lambda)$ defined according to the conditions of the Theorem~\ref{thm:linear_log_cubed_complexity} -- specifically, equal to the smallest positive odd integer satisfying the inequality $2 \log_2 (nx) + 13 < 12 \lambda x$, where $n$ is a linear size of the input grayscale 2D image $I_{n \times n}$, $n \in [2, 4096]$, filled with random integer values.
As theoretically proved, this choice guarantees that the maximal orthotropic approximation error $\mathcal{E}_{SP}(n)=\mathcal{E}_{SP}\left(n, n, \widehat{x}, \widehat{x}, \frac{\widehat{x} - 1}{2}\right)$ of continuous straight lines by discrete $FHT2SP$ patterns is bounded by $\lambda + \frac{1}{2}$.
The relationship between the superpixel size $\widehat{x}$, the meta-parameter $\lambda$, and the linear size of the input square image $n$ is shown in Figure~\ref{fig:superpixel_linear_size_and_normalized_complexity}(a).
The plots in Figure~\ref{fig:superpixel_linear_size_and_normalized_complexity}(b) illustrate the variation in normalized computational complexity $T_{SP}(n) / (n^2 \ln^3 n) \equiv T_{SP}\left(n, n, \widehat{x}, \widehat{x}, \frac{\widehat{x} - 1}{2}\right) / (n^2 \ln^3 n) $ of the $FHT2SP$ algorithm corresponding to the chosen meta-parameter $\lambda$ value.

\begin{figure}[!ht]
\centering

\centering
\begin{minipage}[b]{0.49\linewidth} 
\center{\includegraphics[height=50mm]{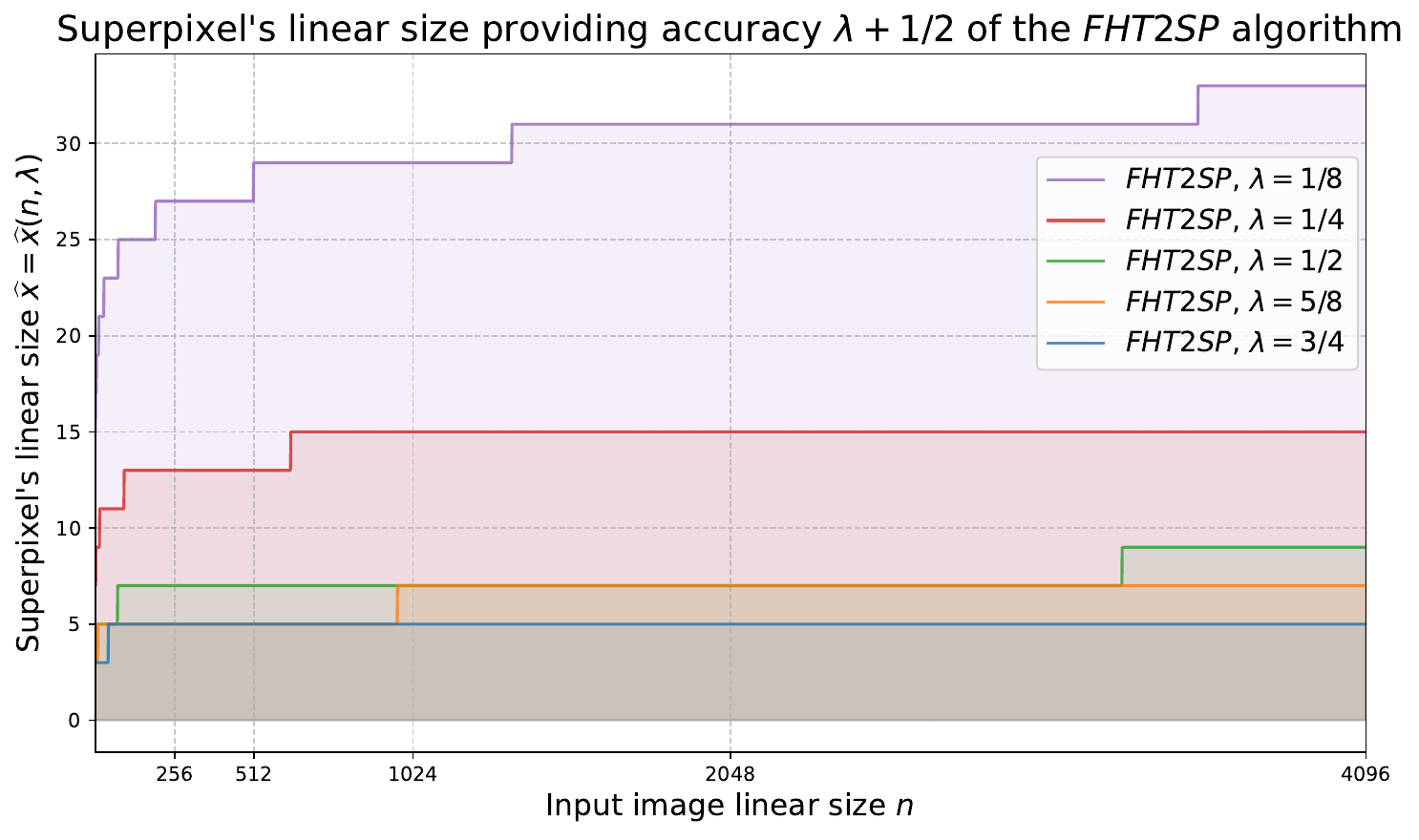}  \\
(a)}
\end{minipage}
\begin{minipage}[b]{0.49\linewidth}
\center{\includegraphics[height=50mm]{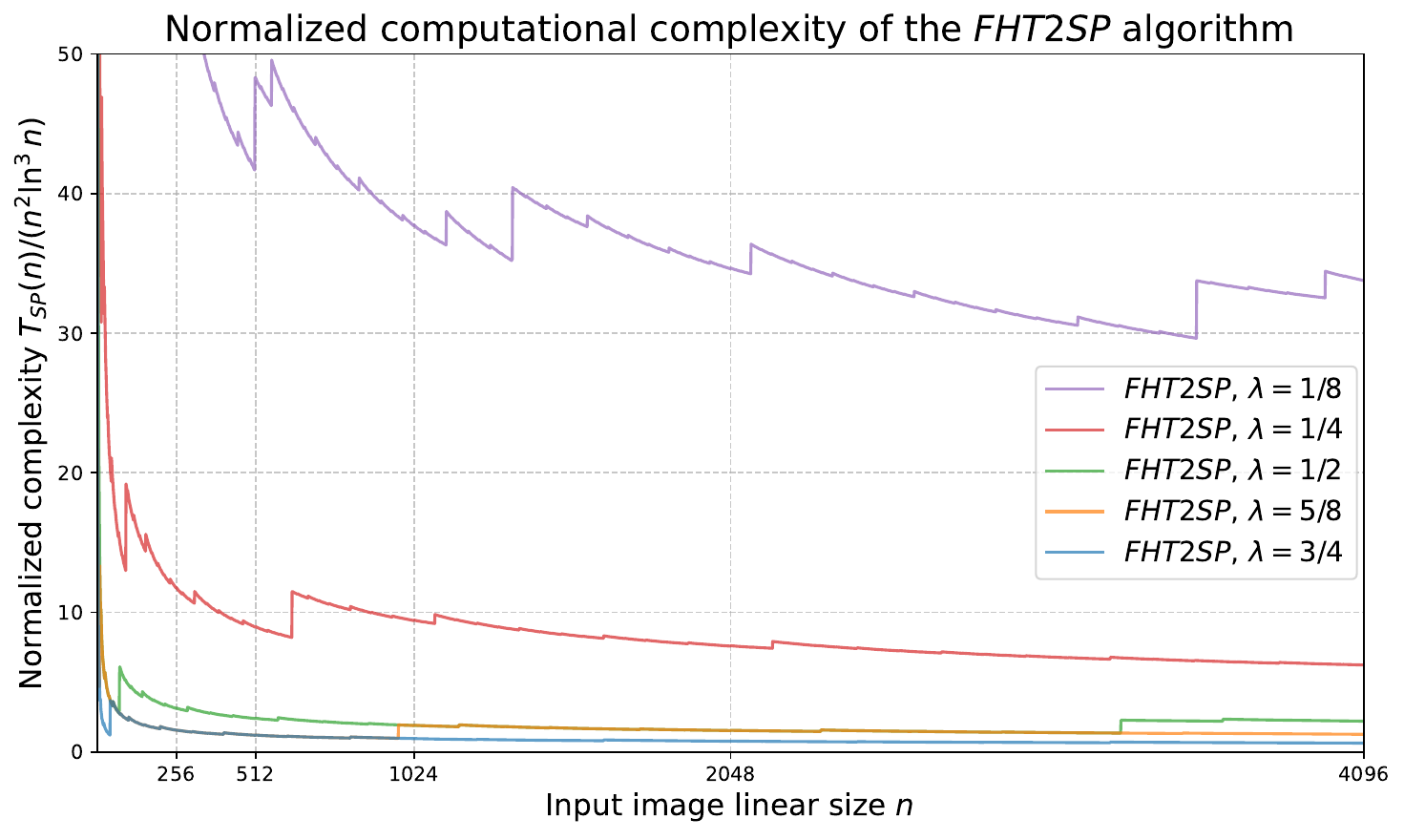} \\
(b) }
\end{minipage}

\caption{ (a) Dependence of the $FHT2SP$ superpixel’s linear size $\widehat{x}$ on the input image linear size $n$ and the meta-parameter $\lambda$ value; (b) Normalized computational complexity $T_{SP}(n) / (n^2 \ln^3 n)$ of the $FHT2SP$ algorithm across different meta-parameter $\lambda$ values.  
\label{fig:superpixel_linear_size_and_normalized_complexity}}
\end{figure}

As the input image size $n$ increases, a larger superpixel is required to achieve the desired approximation accuracy $\lambda + \frac{1}{2}$ (see~Figure~\ref{fig:superpixel_linear_size_and_normalized_complexity}(a)). 
Moreover, the superpixel's linear size $\widehat{x}$ is also inversely proportional to the value of $\lambda$ for any given $n$.
Notably, the discontinuities observed in the plots in Figure~\ref{fig:superpixel_linear_size_and_normalized_complexity}(b) correspond to abrupt changes in the linear size $\widehat{x}$ of the superpixel, which can be viewed in Figure~\ref{fig:superpixel_linear_size_and_normalized_complexity}(a). 
Figure~\ref{fig:superpixel_linear_size_and_normalized_complexity}(b) provides empirical support for Theorem~\ref{thm:linear_log_cubed_complexity} within the considered image size range.
The decreasing trend of the plots indicates a linear-polylogarithmic computational complexity.
Experimental data suggest that, for the given range of $n$, the constant $C(\lambda)$ hidden in the asymptotic complexity $\mathcal{O}(n^2 \ln^3 n)$ can be taken equal to $C(\lambda)=5002.64$, $567.53$, $270.26$, $72.07$, $72.07$ for $\lambda = 1/8, 1/4, 1/2, 5/8, 3/4$, respectively.

We carried out an experimental comparison of the computational accuracy $T_A=T_A(n)$ of various HT algorithms $A$, including the optimally fast $FHT2DT$ and the more accurate, but less computationally efficient, $ASD2$ and $KHM$ algorithms. 
As shown in Figure~\ref{fig:complexity_and_error_comparison}(a), the $FHT2SP$ algorithm with meta-parameter $\lambda = 3/4$, using the superpixel size $\hat{x}$ and nonzero column index $\hat{n} = \frac{\hat{x}-1}{2}$ defined above, outperforms $KHM$ in runtime starting from image linear size $2950$. 
And when $\lambda = 3/4$, the $FHT2SP$'s orthotropic approximation error, theoretically bounded by $5/4$, is comparable to the $1/2$ bound achieved by $ASD2$ and $KHM$. 
Thus, indeed, the $FHT2SP$ algorithm offers a compromise between near-optimal computational complexity and high accuracy (bounded approximation error).

\begin{figure}[!ht]
\centering

\centering
\begin{minipage}[b]{0.49\linewidth}
\center{\includegraphics[height=50mm]{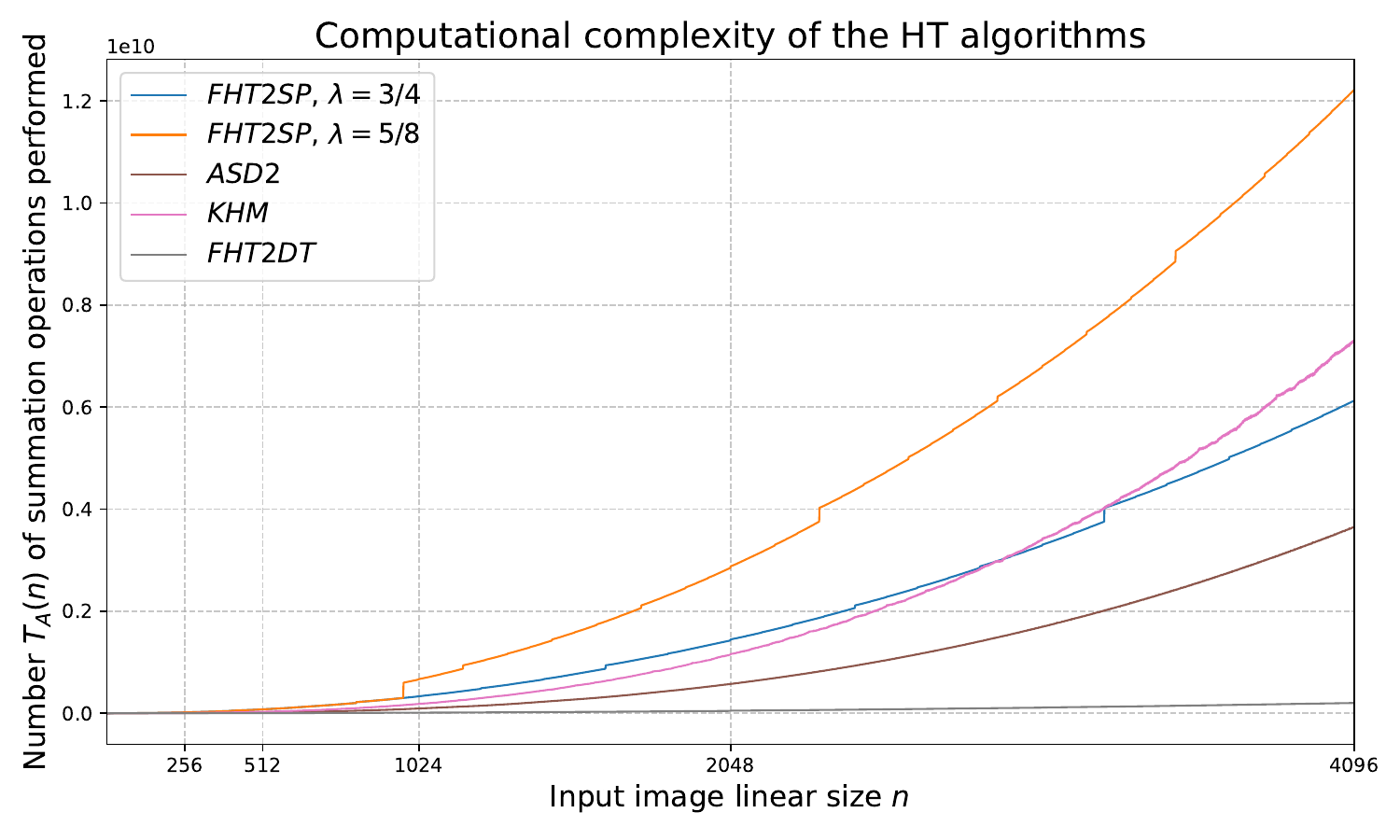} \\
(a)}
\end{minipage}
\begin{minipage}[b]{0.49\linewidth}
\center{\includegraphics[height=50mm]{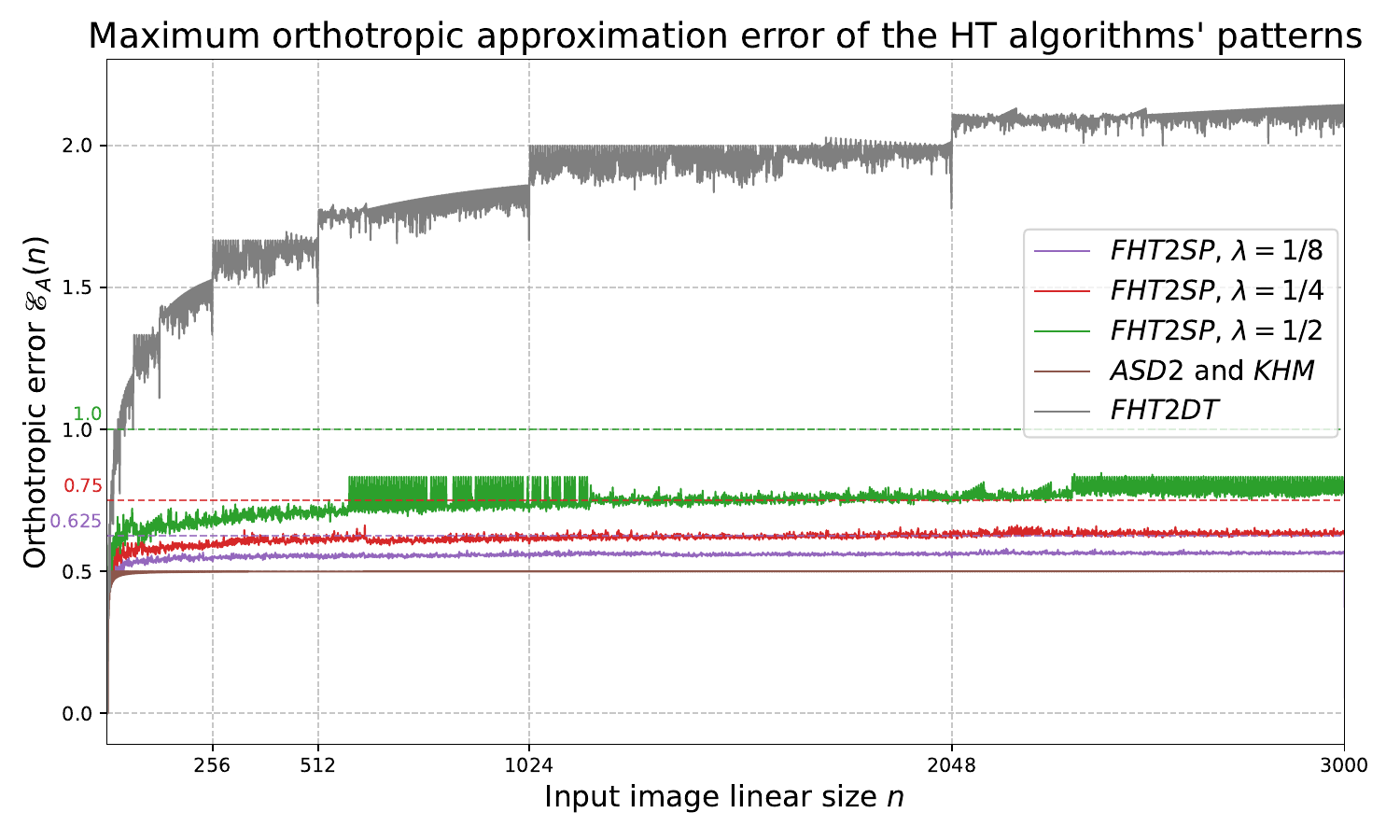} \\
(b) }
\end{minipage}

\caption{ (a) Absolute values of the computational complexity $T_A(n)$ of the compared HT algorithms $A$ for different input image linear sizes $n$; (b) Maximum orthotropic approximation error $\mathcal{E}_A(n)$ provided by the patterns of the compared HT algorithms $A$. 
\label{fig:complexity_and_error_comparison}}
\end{figure}

Within the tested range of $n$, $FHT2SP$ underperforms $ASD2$ in computational complexity for the given $\lambda$ values. 
The smaller the value of the meta-parameter $\lambda$, the greater the gap in computational complexity between $FHT2SP$ and the reference algorithms $ASD2$ and $KHM$ over small values of $n \leq 4096$. 
%within the observed range.
Nevertheless, for larger input image sizes, linear-log-cubed $FHT2SP$ is expected to be significantly faster than $ASD2$ and $KHM$, which both exhibit near-cubic computational complexity. 
Also, it is important to note that the experimentally measured maximum orthotropic approximation error $\mathcal{E}_A=\mathcal{E}_A(n)$ of the compared algorithms $A$, shown in Figure~\ref{fig:complexity_and_error_comparison}(b), confirms that the orthotropic error of the algorithm $FHT2SP$ remains bounded, with a bound independent of the input image size, as theoretically established in Theorem~\ref{thm:fht2sp_accuracy}.

Figure~\ref{fig:Shepp_Logan_sinograms} provides a visual comparison of the accuracy of the evaluated HT algorithms using a test image of the Shepp–Logan phantom~\cite{shepp1974fourier} with a resolution of $3000 \times 3000$ pixels. 
The output of the accurate $ASD2$ and $KHM$ algorithms -- both based on DSLS patterns with a guaranteed orthotropic approximation error bound of $1/2$ --  can be considered as the reference for comparison. 
When applying the fast but significantly less accurate $FHT2DT$ algorithm, noticeable high-frequency artifacts are present in the resulting HT image. 
Unlike $FHT2DT$, the results produced by the algorithm $FHT2SP$ ($\lambda = 1/2$ and $\lambda = 1$) are free of such artifacts and visually indistinguishable from the $ASD2$ and $KHM$ reference.

\begin{figure}[!ht]
\centering

\centering
\begin{minipage}[b]{0.32\linewidth}
\center{\includegraphics[height=58mm]{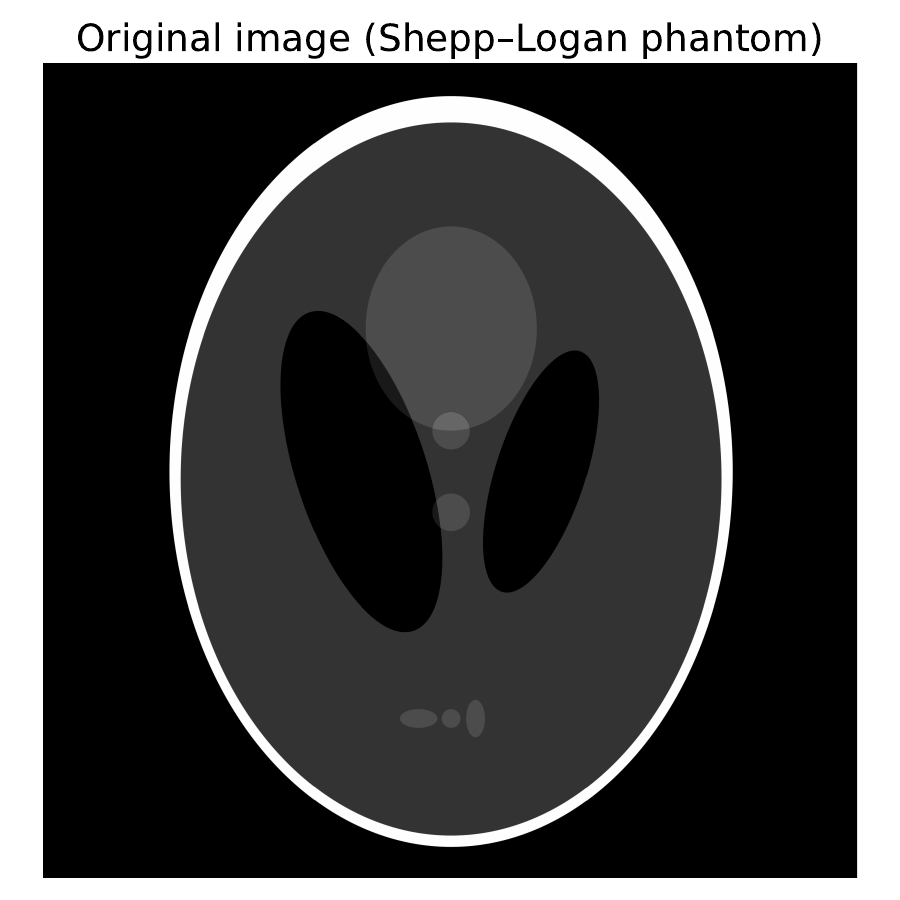} \\
(a)}
\end{minipage}
\begin{minipage}[b]{0.66\linewidth}
\center{\includegraphics[height=59mm]{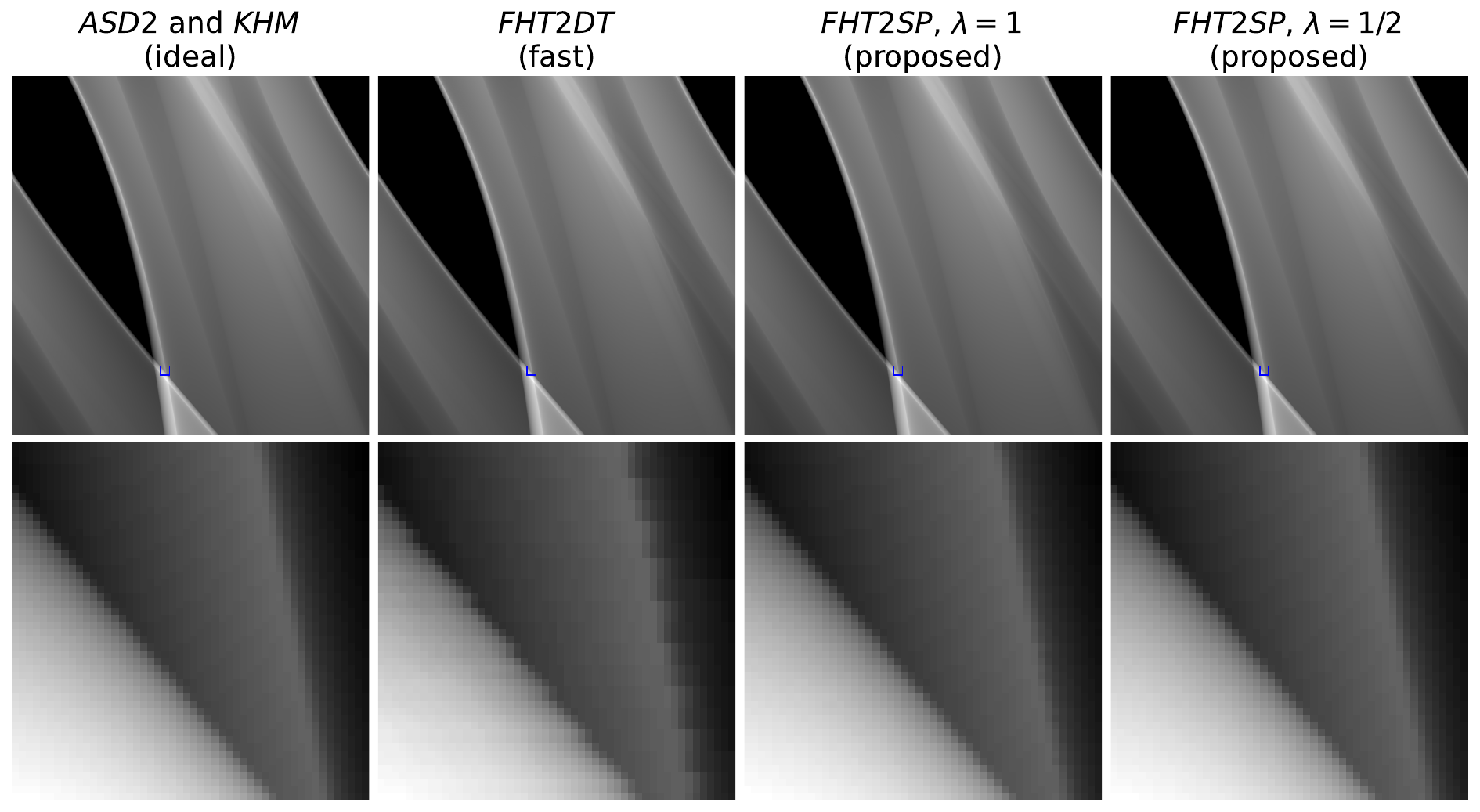} \\
(b) }
\end{minipage}

\caption{An example of the HT output images: (a) original image (Shepp–Logan phantom); (b) the resulting Hough images computed using the compared HT algorithms, with zoomed-in and contrast-enhanced regions indicated by blue rectangles. 
\label{fig:Shepp_Logan_sinograms}}
\end{figure}

A limitation of the $FHT2SP$ algorithm is its high memory consumption, which increases with the superpixel size $\widehat{x}$ and depends on the accuracy-related meta-parameter $\lambda$. 
For instance, achieving an orthotropic error of $5/8$ ($\lambda = 1/8$) on an $1024 \times 1024$ input image requires a $29 \times 29$ superpixel, leading to processing an effective image of size $29696 \times 29696$ via the $FHT2DT$ algorithm. 
A promising way to significantly reduce memory requirements is to incorporate an in-place variant of $FHT2DT$ within the $FHT2SP$ algorithm, highlighting the importance of developing such an in-place version.

\section{Conclusion}
\label{sec:conclusion}

This work puts forward a generalization of Brady’s superpixel concept, initially limited to images with power-of-two dimensions. 
Our extension allows for non-square superpixels with arbitrary width, height, and flexible positioning of the nonzero significant column.
%Although Brady introduced the superpixel concept, the gap in assessing impact of superpixel parameters on HT algorithm performance we conducted.
We thoroughly assessed the impact of both the original superpixels introduced by Brady and their generalized form on HT performance.
Incorporating the generalized superpixel formulation into the $FHT2DT$ FHT algorithm led to the development of a novel $FHT2SP$ HT algorithm.
The $FHT2SP$ algorithm is applicable to arbitrary-shaped images and achieves near-optimal $\mathcal{O}(wh \ln^3 w)$ computational complexity while offering improved accuracy comparable to that of existing accurate, yet asymptotically slower, HT algorithms. 
%The $FHT2SP$ algorithm's patterns approximate continuous lines with a maximum orthotropic error bounded by $\lambda + 1/2$, with $\lambda \in (0,1]$ being a tunable meta-parameter. 
%Theoretical and empirical analyses were conducted on the influence of superpixel parameters and meta-parameter $\lambda$ on the performance characteristics of the $FHT2SP$ algorithm.
%The impact of superpixel parameters and meta-parameter $\lambda$ on $FHT2SP$'s performance was analyzed theoretically and empirically.
%Both the near-optimal complexity and accuracy estimations are established theoretically and validated experimentally.
A Python implementation of the proposed $FHT2SP$ algorithm is open-source and available as a part of the Python library \emph{adrt}~\cite{adrt}.
%, which also contains implementations of $FHT2DT$, $ASD2$, and $KHM$ used for comparison in the paper. 
%The developed $FHT2SP$ algorithm offers significant advantages in practical scenarios that require both fast and accurate HT computation for large industrial images. 
%Specifically, $FHT2SP$ (with $\lambda=1$) demonstrates lower computational complexity than the $KHM$ algorithm for images with linear dimensions exceeding $3000$, and outperforms the $ASD2$ algorithm from linear sizes around $25000$, while achieving nearly the same level of accuracy. 
%In real-world applications, their differences in accuracy are typically negligible.
%The $FHT2SP$ algorithm may be particularly valuable in applications requiring a trade-off between computational efficiency and high accuracy in HT-based processing of large images, including those with non-power-of-two dimensions.
$FHT2SP$ provides clear benefits for fast and accurate HT on large industrial images, specifically, when $\lambda=1$, being faster than $KHM$ beyond 3000 pixels and than $ASD2$ beyond around 25000 pixels, with accuracy close enough to make the difference negligible in real use.

\bibliographystyle{spiebib}
%\bibliography{bibliography.tex}
%\input{bibliography.tex}
\bibliography{main}

\end{document}